%% file: main.tex
\documentclass[11pt]{article}
\usepackage{enumerate}

\usepackage[
            CJKbookmarks=true,
            bookmarksnumbered=true,
            bookmarksopen=true,
                       bookmarks=true,
            colorlinks=true,
            citecolor=red,
            linkcolor=blue,
            anchorcolor=red,
            urlcolor=blue
            ]{hyperref}

\usepackage{url, color, fullpage,xcolor}
\usepackage{amsmath,amssymb,amsfonts,amsthm,bbm,graphicx,subfigure,xspace}
\usepackage{tcolorbox}
\definecolor{myblue}{HTML}{0000FF}
\definecolor{myred}{HTML}{FF0000}
\definecolor{myorange}{HTML}{FFA500}
\definecolor{myyellow}{HTML}{FFFF00}
\definecolor{mycyan}{HTML}{00FFFF}
\definecolor{mygreen}{HTML}{008000}
\definecolor{mybrown}{HTML}{A52A2A}

\input{def_t}

\begin{document}
\title{Distributed Mean Estimation with Limited Communication}

\author{
Ananda Theertha Suresh, Felix X. Yu, Sanjiv Kumar, H. Brendan McMahan\\ \\
  \texttt{\{theertha, felixyu, sanjivk, mcmahan\}@google.com}
}
\maketitle
\begin{abstract}
Motivated by the need for distributed learning and optimization algorithms with low communication cost, we study communication efficient algorithms for distributed mean estimation. Unlike previous works, we make no probabilistic assumptions on the data. We first show that for $d$ dimensional data with $n$ clients, a naive stochastic binary rounding approach yields a mean squared error (MSE) of $\Theta(d/n)$ and uses a constant number of bits per dimension per client. We then extend this naive algorithm in two ways: we show that applying a structured random rotation before quantization reduces the error to $\mathcal{O}((\log d)/n)$ and a better coding strategy further reduces the error to $\mathcal{O}(1/n)$  and uses a constant number of bits per dimension per client. We also show that the latter coding strategy  is optimal up to a constant in the minimax sense i.e., it achieves the best MSE for a given communication cost. We finally demonstrate the practicality of our algorithms by applying them to distributed Lloyd's algorithm for k-means and power iteration for PCA.
  \end{abstract}
\section{Introduction}
\subsection{Background}
Given $n$ vectors $X^n \ed X_1,X_2\ldots, X_n \in\RR^d$ that reside
on $n$ clients, the goal of \emph{distributed mean estimation} is to
estimate the mean of the vectors:
\begin{equation}
  \label{eq:mean}
\bar{X} \ed \frac{1}{n} \sum^{n}_{i=1} X_i.
\end{equation}
This basic estimation problem is used as a subroutine in several
learning and optimization tasks where data is distributed across
several clients. For example, in Lloyd's
algorithm~\cite{Lloyd82} for k-means clustering, if data is distributed across several
clients, the server needs to compute the means of all clusters in each
update step. Similarly, for PCA, if data samples are distributed
across several clients, then for the power-iteration method, the
server needs to average the output of all clients in each step.

Recently, algorithms involving distributed mean estimation have been used extensively in training large-scale
neural networks and other statistical models \cite{McdonaldGM10, PoveyZK14, dean2012large,
  mcmahan2016federated, AlistarhLTV16}. In a typical scenario of
synchronized distributed learning, each client obtains a copy of a 
global model. The clients then update the model independently based on
their local data.  The updates (usually in the form of
gradients) are then sent to a server, where they are averaged and
used to update the global model.  A critical step in all of the
above algorithms is to estimate the mean of a set of vectors as in Eq.~\eqref{eq:mean}.

One of the main bottlenecks in distributed algorithms is the
communication cost. This has spurred a line of work focusing on communication
cost in learning ~\cite{Tsitsiklis87,BalcanBFM12,ZhangYDJW13,ArjevaniS15,ChenSWZ16}. The
communication cost can be prohibitive for modern applications, where
each client can be a low-power and low-bandwidth device such as a mobile 
phone~\cite{konevcnyMYPSB16}. Given such a wide set of applications, we study the basic problem of achieving the optimal minimax rate in distributed mean
estimation with limited communication.

We note that our model and results differ from previous works on mean
estimation~\cite{ZhangYDJW13, GargMN14,BravermanGMNW16} in two ways:
previous works assume that the data is generated i.i.d. according to some
distribution; we do not make any distribution assumptions on
data. Secondly, the objective in prior works is to estimate the mean of the
underlying statistical model; our goal is to estimate the
empirical mean of the data.

\subsection{Model}
Our proposed communication algorithms are simultaneous and independent,
i.e., the clients independently send data to the server and they can
transmit at the same time.  In any independent communication protocol, each client
transmits a function of $X_i$ (say $f(X_i)$), and a central server
estimates the mean by some function of $f(X_1),f(X_2),\ldots,
f(X_n)$. Let $\pi$ be any such protocol and let $\cC_i(\pi, X_i)$ be
the expected number of transmitted bits by the $i$-th client during protocol
$\pi$, where throughout the paper, expectation is over the randomness in
protocol $\pi$.

The total number of bits transmitted by all clients with the protocol $\pi$ is 
\[
\cC(\pi, X^n) \ed  \sum^n_{i=1} \cC_i(\pi, X_i).
\]
Let the estimated mean be $\hat{\bar{X}}$. For a protocol $\pi$, the MSE of
the estimate is
\[
\cE(\pi, X^n) = \EE\left[\norm{\hat{\bar{X}} - \bar{X}}^2_2\right].
\]

We allow the use of both private and public randomness.
Private randomness refers to random values that are generated by each
machine separately, and public randomness refers to a sequence of
random values that are shared among all parties\footnote{In the absence of
  public randomness, the server can communicate a random seed that can
  be used by clients to emulate public randomness.}.

The proposed algorithms work for any $X^n$. To measure the
minimax performance, without loss of generality, we restrict ourselves to the scenario where each
$X_i \in S^d$, the ball of radius $1$ in $\RR^d$, i.e., $X \in S_d$ iff
\[
\norm{X}_2 \leq 1,
\]
where $\norm{X}_2$ denotes the $\ell_2$ norm of the vector $X$. For a
protocol $\pi$, 
 the worst case error for all $X^n \in S^d$ is
\[
\cE(\pi, S^d) \ed  \max_{X^n: X_i \in S^d\, \forall i} \cE(\pi, X^n).
\]
Let $\Pi(c)$ denote the set of all protocols with communication cost at most $c$.
The minimax MSE is 
\[
\cE(\Pi(c), S^d) \ed \min_{\pi \in \Pi(c)} \cE(\pi, S^d).
\]

\subsection{Results and discussion}
\subsubsection{Algorithms}
We first analyze the MSE $\cE(\pi, X^n)$ for three
algorithms, when $\cC(\pi, X^n) = \Theta(n d)$, i.e., \textit{each
  client sends a constant number of bits per dimension}.

\textbf{Stochastic uniform quantization. }
In Section~\ref{sec:sbq}, as a warm-up we first show that a naive
stochastic binary quantization algorithm (denoted by $\pi_{sb}$) achieves an MSE
of
\[
\cE(\pi_{sb}, X^n) = \Theta \left(\frac{d}{n} \cdot \frac{1}{n} \sum^n_{i=1}
\norm{X_i}^2_2\right),
\]
and $\cC(\pi_{sb}, X^n) = n \cdot (d + \tilde{\cO}(1))$, i.e., \textit{each client
  sends one bit per dimension}.\footnote{We
  use $\tilde{\cO}(1)$ to denote $\cO(\log(dn))$. 
} We further show that this bound is
tight. In many practical scenarios, $d$ is much larger than $n$ and the
above error is prohibitive~\cite{konevcnyMYPSB16}.

A natural way to decrease the error is to increase the number of levels
of quantization. If we use $k$ levels of quantization, in Theorem~\ref{thm:sk}, we show that the
error decreases as
\begin{equation}
\label{eq:esk}
\cE(\pi_{sk}, X^n) = \cO \left(\frac{d}{n(k-1)^2} \cdot \frac{1}{n} \sum^n_{i=1}
\norm{X_i}^2_2\right).
\end{equation}
However, the communication cost would increase to $\cC(\pi_{sk}, X^n) = n \cdot (d \lceil
\log_2 k \rceil + \tilde{\cO}(1))$ bits, which can be expensive, if
we would like the MSE to be $o(d/n)$.

In order to reduce the communication cost, we propose two approaches.

\textbf{Stochastic rotated quantization:} We show that
preprocessing the data by a random rotation reduces the mean squared
error. Specifically, in Theorem~\ref{thm:rhd}, we show that this new
scheme (denoted by $\pi_{srk}$) achieves an MSE of
 \[
\cE(\pi_{srk}, X^n) = \cO \left(\frac{\log d}{n(k-1)^2}  \cdot \frac{1}{n} \sum^n_{i=1}
\norm{X_i}^2_2 \right).
\]
Note that throughout the paper, all logarithms are to base $e$, unless stated.
The scheme has a communication cost of $\cC(\pi_{srk}, X^n) = n \cdot (d \lceil \log_2 k
\rceil +\tilde{\cO}(1))$. Note
that the new scheme achieves much smaller MSE than naive
stochastic quantization for the same communication cost.

\textbf{Variable length coding:} Our second approach uses the same quantization as $\pi_{sk}$ but encodes levels via variable
length coding.  Instead of using $ \lceil \log_2 k \rceil$ bits per
dimension, we show that using variable length encoding such as
arithmetic coding to compress the data reduces the communication cost significantly. In particular, in Theorem~\ref{thm:arith} we show that
there is a scheme (denoted by $\pi_{svk}$) such that 
\begin{equation}
\label{eq:cvar}
\cC(\pi_{svk},X^n) = \cO(nd(1 + \log(k^2/d + 1)) +  \tcO(n)),
\end{equation}
and $\cE(\pi_{svk}, X^n) = \cE(\pi_{sk}, X^n)$.
Hence, 
setting $k = \sqrt{d}$ in Eqs.~\ref{eq:esk} and~\ref{eq:cvar} yields 
\[
\cE(\pi_{svk}, X^n) =  \cO \left(\frac{1}{n} \cdot \frac{1}{n} \sum^n_{i=1}
 \norm{X_i}^2_2\right),
 \]
 and with $\Theta(nd)$ bits of communication i.e., constant number of bits per
 dimension per client. Of the three protocols, $\pi_{svk}$ has the best MSE for a given communication cost. Note that $\pi_{svk}$ uses $k$ quantization levels but still uses $\cO(1)$ bits per dimension per client for all $k \leq \sqrt{d}$.

Theoretically, while variable length coding has better guarantees, stochastic rotated quantization has several practical advantages: it uses fixed length coding and hence 
can be combined with encryption schemes for privacy preserving secure aggregation~\cite{Bonawitz16}. It can also provide lower quantization error in some scenarios due to better constants (see Section~\ref{sec:experiments} for details).

Concurrent to this work, \cite{AlistarhLTV16} showed that stochastic quantization and Elias coding~\cite{elias1975} can be used to obtain communication-optimal SGD. Recently, \cite{konevcnyR16} showed that $\pi_{sb}$ can be improved further by optimizing the choice of stochastic quantization boundaries. However, their results depend on the number of bits necessary to represent a float,
whereas ours do not. 

\subsubsection{Minimax MSE}
In the above protocols, all of the clients transmit the data. We
augment these protocols with a sampling procedure, where only a random
fraction of clients transmit data. We show that a combination of $k$-level
quantization, variable length coding, and sampling can be used to achieve
information theoretically optimal MSE for a given
communication cost.  In particular, combining
Corollary~\ref{cor:upper} and Theorem~\ref{thm:lower} yields our minimax
result:
\begin{Theorem}
  \label{thm:main} There exists a universal constant $t < 1$ such that
  for communication cost $c \leq ndt$ and $n \geq 1/t$,
\[
\cE(\Pi(c), S^d) = \Theta\left(\min\left(1, \frac{d}{c}\right)\right).
\]
\end{Theorem}
This result shows that the product of communication cost and MSE scales linearly in the number of dimensions.

The rest of the paper is organized as follows. We first analyze the
stochastic uniform quantization technique in
Section~\ref{sec:sknaive}. In Section~\ref{sec:rsbq}, we propose the
stochastic rotated quantization technique, and in Section~\ref{sec:sk}
we analyze arithmetic coding.  In Section~\ref{sec:sampling}, we
combine the above algorithm with a sampling technique and state the
upper bound on the minimax risk, and in Section~\ref{sec:lower} we
state the matching minimax lower bounds. Finally, in
Section~\ref{sec:experiments} we discuss some practical considerations and apply these algorithms on distributed power iteration and Lloyd's algorithm.

\section{Stochastic uniform quantization}
\label{sec:sknaive}
\subsection{Warm-up: Stochastic binary quantization}
\label{sec:sbq}
For a vector $X_i$, let $X^{\max}_i = \max_{1 \leq j \leq d} X_i(j)$
and similarly let $X^{\min}_i = \min_{1 \leq j \leq d} X_i(j)$.  In
the stochastic binary quantization protocol $\pi_{sb}$, for each
client $i$, the quantized value for each coordinate $j$ is generated
independently with private randomness as
\[
Y_i(j) = 
\begin{cases}
    X^{\max}_i & \text{w.p. } \frac{X_i(j) - X^{\min}_i}{X^{\max}_i - X^{\min}_i}, \\
    X^{\min}_i & \text{otherwise.}
\end{cases}
\]
Observe $\EE Y_i(j) = X_i(j)$.
The server estimates $\bar{X}$ by
\[
\hat{\bar{X}}_{\pi_{sb}} = \frac{1}{n} \sum^n_{i=1} Y_i.
\]
We first bound the communication cost of this protocol.
\begin{Lemma}
  \label{lem:sb_com}
  There exists an implementation of \emph{stochastic binary
    quantization} that uses $d + \tilde{\cO}(1)$ bits per client and hence
  $\cC(\pi_{sb}, X^n) \leq n \cdot \left( d + \tilde{\cO}(1) \right)$. 
\end{Lemma}
\begin{proof}
   Instead of sending vectors $Y_i$, clients transmit two real
   values $X^{\max}_i$ and $X^{\min}_i$ (to a desired precision) and a
    bit vector $Y'_i$ such that $Y'_i(j) = 1$ if $Y_i =
   X^{\max}_i$ and $0$ otherwise. Hence each client transmits $d + 2r$
   bits, where $r$ is the number of bits to transmit the real value to
   a desired precision.

Let $N$ be the maximum norm of the underlying vectors.
   To bound $r$, observe that using $r$ bits, one can represent a
   number between $-N$ and $N$ to an error of $N/2^{r-1}$. Thus using
   $3 \log_2 (dn) + 1$ bits one can represent the minimum and maximum to
   an additive error of $N/(nd)^3$. This error in transmitting minimum
   and maximum of the vector does not affect our calculations and we
   ignore it for simplicity.    We note that in practice, each dimension of $X_i$ is often stored
   as a 32 bit or 64 bit float, and $r$ should be
   set as either 32 or 64. In this case, using an even larger $r$
    does not further reduce the error.
\end{proof}
\vspace{-0.2cm}
We now compute the estimation error of this protocol.
\begin{Lemma}
  \label{lem:general}
  For any set of vectors $X^n$, 
  \[
\cE(\pi_{sb}, X^n) = \frac{1}{n^2} \sum^n_{i=1}\sum^d_{j=1}
(X^{\max}_i - X_i(j))(X_i(j)- X^{\min}_i) .
  \]
  \end{Lemma}
\begin{proof}
\begin{align*}
\cE(\pi_{sb}, X^n) & = \EE \norm{\hat{\bar{X}} - \bar{X}}^2_2 
= \frac{1}{n^2} \EE\norm{\sum^n_{i=1} (Y_i -X_i)}^2_2 
\\&  = \frac{1}{n^2} \sum^n_{i=1} \EE \norm{Y_i - X_i}^2_2,
\end{align*}
where the last equality follows by observing that $Y_i - X_i$, $\forall i$, are independent zero mean random variables. 
The proof follows by
observing that for every $i$,
\begin{align}
\EE \norm{Y_i - X_i}^2_2
 & = \sum^d_{j=1} \EE[(Y_i(j) - X_i(j))^2] \notag \\ 
 & = \sum^d_{j=1} (X^{\max}_i - X_i(j))(X_i(j) - X^{\min}_i). \qedhere
  \nonumber
\end{align}
\end{proof}

Lemma~\ref{lem:general} implies the following upper bound.
\begin{Lemma}
  \label{lem:general_upper}
    For any set of vectors $X^n$, 
  \[
  \cE(\pi_{sb}, X^n)
\leq \frac{d}{2n} \cdot \frac{1}{n} \sum^n_{i=1}
\norm{X_i}^2_2.
  \]
\end{Lemma}
\begin{proof}
  The proof follows by Lemma~\ref{lem:general} observing that $\forall j$
\begin{equation*}
(X^{\max}_i - X_i(j))(X_i(j) - X^{\min}_i) \leq \frac{(X^{\max}_i -
  X^{\min}_i)^2}{4},
\end{equation*}
and
\begin{equation}
  \label{eq:normbound}
(X^{\max}_i - X^{\min}_i)^2 
\leq 2 \norm{X_i}^2_2. \qedhere
\end{equation}
\end{proof}
We also show that the above bound is tight:
\begin{Lemma}
  \label{lem:lower_naive}
  There exists a set of vectors $X^n$ such that
  \[
\cE(\pi_{sb}, X^n) \geq \frac{d-2}{2n} \cdot \frac{1}{n} \sum^n_{i=1} \norm{X_i}^2_2.
  \]
  \end{Lemma}
\begin{proof}
  For every $i$, let $X_i$ be defined as follows. $X_i(1) =
  1/\sqrt{2}$, $X_i(2) = -1/\sqrt{2}$, and for all $j > 2$, $X_i(j) =
  0$. For every $i$, $X^{\max}_i = \frac{1}{\sqrt{2}}$ and $X^{\min}_i
  = -\frac{1}{\sqrt{2}}$.  Substituting these bounds in the conclusion
  of Lemma~\ref{lem:general} (which is an equality) yields the theorem.
  \end{proof}

Therefore, the simple algorithm proposed in this section gives MSE $\Theta (d/n)$ times the average norm. Such an error is too large for real-world use. For example, in the application of neural networks \cite{konevcnyMYPSB16}, $d$ can be on the order of millions, yet $n$ can be much smaller than that. In such cases, the MSE is even larger than the norm of the vector.

\subsection{Stochastic $k$-level quantization}
\label{sec:sk_2}

A natural generalization of binary quantization is $k$-level
quantization.  Let $k$ be a positive integer larger than $2$. We propose a
\emph{$k$-level stochastic quantization} scheme $\pi_{sk}$ to quantize each coordinate.  Recall that for a vector $X_i$, $X^{\max}_i =
\max_{1 \leq j \leq d} X_i(j)$ and $X^{\min}_i =
\min_{1 \leq j \leq d} X_i(j)$.  For every integer $r$ in the range $[0,k)$,
  let
\[
B_i(r) \ed X^{\min}_i +  \frac{rs_i}{k-1},
\]
where $s_i$ satisfies $X^{\min}_i + s_i \geq X^{\max}_i$. A natural choice
for $s_i$ would be $X^{\max}_i - X^{\min}_i$.\footnote{We will show in Section \ref{sec:sk}, however, a higher value of $s_i$ and variable
length coding has better guarantees.} The algorithm
quantizes each coordinate into one of $B_i(r)$s stochastically. 
In $\pi_{sk}$, for the $i$-th datapoint and $j$-th coordinate, if $X_i(j) \in [B_i(r), B_i({r+1}))$,
\[
Y_i(j) = 
\begin{cases}
    B_i({r+1}) & \text{w.p. } \frac{X_i(j) - B_i(r)}{B_i(r+1) - B_i(r)} \\
   B_i(r) & \text{otherwise.}
\end{cases}
\]
The server estimates $\bar{X}$ by
\[
\hat{\bar{X}}_{\pi_{sk}} = \frac{1}{n} \sum^n_{i=1} Y_i.
\]
As before, the communication complexity of this protocol is bounded. The proof is similar to that of~Lemma~\ref{lem:sb_com} and hence omitted.
\begin{Lemma}
  \label{lem:comrsb}
 There exists an implementation of \emph{stochastic $k$-level quantization}
  that uses $d \lceil\log(k)\rceil + \tilde{\cO}(1)$ bits per client and hence
  $\cC(\pi_{sk}, X^n) \leq n \cdot \left( d \lceil \log_2 k \rceil + \tilde{\cO}(1) \right)$.
  \end{Lemma}
The mean squared loss can be bounded as follows.
\begin{Theorem}
  \label{thm:sk}
If $X^{\max}_i - X^{\min}_i \le s_i \leq \sqrt{2} \norm{X_i}_2 \,\,\forall i$, then for any $X^n$, the $\pi_{sk}$ protocol satisfies,
   \[
\cE(\pi_{sk}, X^n) \leq \frac{d}{2n(k-1)^2}  \cdot \frac{1}{n} \sum^n_{i=1}
\norm{X_i}^2_2.
\]
\end{Theorem}
\begin{proof}
\begin{align}
&  \cE(\pi_{sk}, X^n)  = \EE \norm{\hat{\bar{X}} - \bar{X}}^2_2 
= \frac{1}{n^2} \EE \norm{\sum^n_{i=1} (Y_i -X_i)}^2_2  \nonumber
\\ & = \frac{1}{n^2} \sum^n_{i=1} \EE \norm{Y_i - X_i}^2_2
     \leq \frac{1}{n^2}  \sum^n_{i=1} d \frac{s^2_i}{4(k-1)^2} \label{eq:kbound},
\end{align}
where the last equality follows by observing $Y_i(j) - X_i(j)$ is an
independent zero mean random variable with $\EE(Y_i(j) - X_i(j))^2 \le \frac{s^2_i}{4(k-1)^2}$. $s_i \leq \sqrt{2} \norm{X_i}_2$ completes the proof.
\end{proof}
\vspace{-0.2cm}
We conclude this section by noting that $s_i = X^{\max}_i - X^{\min}_i$
satisfies the conditions for the above theorem by
Eq.~\eqref{eq:normbound}.
\section{Stochastic rotated quantization}
\label{sec:rsbq}
We show that the algorithm of the previous section can be
significantly improved by a new protocol. 
The motivation comes from the fact that the MSE of stochastic binary quantization and stochastic $k$-level quantization is $O(\frac{d}{n}(X_i^{\max} - X_i^{\min})^2)$ (the proof of Lemma \ref{lem:general_upper} and Theorem \ref{thm:sk} with $s_i = X_i^{\max} - X_i^{\min}$). Therefore the MSE is smaller when $X_i^{\max}$ and $X_i^{\min}$ are close. 
For example, when $X_i$ is generated uniformly on the unit sphere, with high probability, $X_i^{\max} - X_i^{\min}$ is $\cO\left(\sqrt{\frac{\log d} {d}}\right)$ \cite{dasgupta2003elementary}. In such case, $\cE(\pi_{sk}, X^n)$ is $\cO(\frac{\log d}{n})$ instead of $\cO(\frac{d}{n})$.

In this section, we show that even without any assumptions on the distribution of the data, we can ``reduce'' $X_i^{\max} - X_i^{\min}$ with a \textit{structured random rotation}, yielding an $\cO(\frac{\log d}{n})$ error. We call the method \textit{stochastic rotated quantization} and denote it by $\pi_{srk}$.

Using public randomness, all clients and the central server generate a
random rotation matrix (random orthogonal matrix) $R \in \RR^{d\times d}$ according to some known distribution.
Let $Z_i = R X_i$ and $\bar{Z} = R \bar{X}$. In the stochastic rotated
quantization protocol $\pi_{srk}(R)$, clients quantize the vectors $Z_i$ instead of $X_i$ and transmit them similar to $\pi_{srk}$.
The server estimates $\bar{X}$ by
\[
\hat{\bar{X}}_{\pi_{srk}} = R^{-1} \hat{\bar{Z}},  \quad \hat{\bar{Z}} = \frac{1}{n} \sum^n_{i=1} Y_i.
\]
The communication cost is same as $\pi_{sk}$ and is given by Lemma~\ref{lem:comrsb}. We now bound the MSE.
\begin{Lemma}
 For any $X^n$, $\cE(\pi_{srk}(R), X^n)$ is at most
  \begin{align*}
\frac{d}{2n^2 (k-1)^2} \sum^n_{i=1} \EE_{R}\left[
  \left(\maxZi\right)^2 +   \left(\minZi\right)^2\right],
\end{align*}
where $Z_i = R X_i$ and for every $i$, let $s_i = Z^{\max}_i - Z^{\min}_i$. 
\end{Lemma}
\begin{proof}
\begin{align*}
\cE(\pi_{srk}, X^n) 
  & = \EE_{\pi} \norm{\hat{\bar{X}} - \bar{X}}^2   \\
&   = \EE_{\pi} \norm{R^{-1}\hat{\bar{Z}} - R^{-1}\bar{Z}}^2 \stackrel{(a)}{=} \EE_{\pi} \norm{\hat{\bar{Z}} - \bar{Z}}^2 \\ 
  &    \stackrel{(b)}{=} \EE_{\pi} \left[  \norm{\hat{\bar{Z}} - \bar{Z}}^2 | Z^n \right]\\ 
  & \leq \frac{d}{4n^2(k-1)^2} \sum^n_{i=1} \EE_R[(\maxZi - \minZi)^2],
\end{align*}
where the last inequality follows Eq.~\eqref{eq:kbound} and the value of $s_i$.
$(a)$ follows from the fact that the rotation does not change the norm of the
vector, and $(b)$ follows from the tower law of expectation.  
The lemma
follows from observing that
\[
(\maxZi - \minZi)^2 \leq 2(\maxZi)^2 + 2 (\minZi)^2. \qedhere
\]

\end{proof}
To obtain strong bounds, we need to find an orthogonal matrix $R$ that
achieves low $(\maxZi)^2$ and $(\minZi)^2$.  In addition, due to the fact that $d$
can be huge in practice, we need a type of orthogonal matrix that permits
fast matrix-vector products. Naive orthogonal matrices that
support fast multiplication such as block-diagonal matrices often
result in high values of $(\maxZi)^2$ and $(\minZi)^2$.  
Motivated by recent works
of structured matrices~\cite{AilonL09, yu2016orthogonal}, we propose to use a special type of orthogonal matrix $R =
HD$, where $D$ is a random diagonal matrix
with \iid \ Rademacher entries ($\pm 1$ with probability $0.5$). $H$ is a Walsh-Hadamard matrix \cite{horadam2012hadamard}.
The Walsh-Hadamard matrix of dimension $2^m$ for $m \in \mathcal{N}$ is given by the recursive formula, 
\begin{align*}
H(2^1) = \begin{bmatrix}
1 & 1\\
1 & -1\end{bmatrix}, 
H(2^m) = \begin{bmatrix}
H(2^{m-1}) &  H(2^{m-1})\\
H(2^{m-1})  & -H(2^{m-1})\end{bmatrix}.
\end{align*}

Both applying the rotation and
inverse rotation take $\cO(d \log d)$ time and $\cO(1)$ additional space
(with an in-place algorithm).  The next lemma bounds $\EE
  \left(\maxZi\right)^2$  and $\EE
  \left(\minZi\right)^2$ for this choice
of $R$. The lemma is similar to that of \cite{AilonL09}, and we
give the proof in Appendix~\ref{app:maxz} for completeness.
\begin{Lemma}
\label{lem:maxz}
Let $R= HD$, where $D$ is a diagonal matrix with independent Radamacher
random variables. For every $i$ and every sequence $X^n$,
  \[
  \EE\left[(\minZi)^2\right]= \EE\left[(\maxZi)^2\right]
  \leq \frac{\norm{X_i}^2_2 (2\log d + 2)}{d}.
  \]
\end{Lemma}
Combining the above two lemmas yields the main result.
\begin{Theorem}
  \label{thm:rhd}
  For any $X^n$, $\pi_{srk} (HD)$ protocol satisfies,
 \[
\cE(\pi_{srk}(HD), X^n) \leq \frac{2\log d + 2}{n(k-1)^2} \cdot \frac{1}{n}
\sum^n_{i=1}\norm{X_i}^2_2.
  \]
\end{Theorem}
\section{Variable length coding}
\label{sec:sk}

Instead of preprocessing the data via a rotation matrix as in
$\pi_{srk}$, in this section we propose to use a variable length coding strategy to
minimize the number of bits.

Consider the stochastic $k$-level quantization technique.
A natural way of transmitting $Y_i$ is sending the bin number for each
coordinate, thus the total number of bits the algorithm sends per
transmitted coordinate would be $d\lceil \log_2 k
\rceil$. This naive implementation is sub-optimal.  Instead, we propose to further encode the transmitted values
using universal compression schemes~\cite{KrichevskyT81,
  FalahatgarJOPS15}. We first encode $h_r$, the number of times each
quantized value $r$ has appeared, and then use arithmetic or Huffman
coding corresponding to the distribution 
$
p_r = \frac{h_r}{d}.
$
We denote this scheme by $\pi_{svk}$. Since we quantize vectors the same way in $\pi_{sk}$ and $\pi_{svk}$, the MSE of $\pi_{svk}$ is also given by Theorem~\ref{thm:sk}. We now bound the communication cost.
\begin{Theorem}
  \label{thm:arith}
Let $s_i = \sqrt{2} \norm{X_i}$.  There exists an implementation of $\pi_{svk}$ such that $\cC(\pi_{svk}, X^n)$ is at most
\begin{small}
\[
n  \left(d \left(2+ \log_2\left(\frac{(k-1)^2}{2d }+ \frac{5}{4} \right) \right)+ k  \log_2\frac{(d
  + k)e}{k} + \tilde\cO(1) \right).
  \]
  \end{small}
\end{Theorem}
\begin{proof}
As in Lemma \ref{lem:sb_com}, $\tilde\cO(1)$ bits are used to transmit the $s_i$'s and $X^{\min}_i$.
  Recall that $h_r$ is the number of coordinates that are quantized
  into bin $r$, and $r$ takes $k$ possible values. Furthermore,
 $\sum_r h_r = d$ and for each $r$, $h_r \geq 0$. Thus the number of bits necessary to represent the $h_r$'s is
  \[
 \left\lceil \log_2 {d + k - 1 \choose k-1 } \right \rceil \leq k \log_2 \frac{(d+k)e}{k}.
  \]
 Once we have compressed the $h_r$'s, we use arithmetic coding corresponding to the distribution $p_r = h_r/d$
  to compress and transmit bin values for each coordinate.  The total number of bits arithmetic coding uses is \cite{Mackay03}
\[
d \sum^{k-1}_{r=0} \frac{h_r}{d} \log_2 \frac{d}{h_r} + 2.
\]
Let $a =(k-1) X^{\min}_i$ and $b = s_i$. Note that if $Y_i(j)$ belongs to bin $r$, $(a+br)^2 = (k-1)^2 Y^2_i(j)$ and $h_r$ is the number of coordinates quantized into bin $r$.
Hence $\sum_r h_r (a + b r)^2$  is the scaled norm-square of $Y_i$, i.e.,
\begin{align*}
\sum_r h_r (a+br)^2 
& = (k-1)^2 \sum^d_{j=1} Y^2_i(j) \\
& = \sum^d_{j=1} \left( {(X_i(j)+ \alpha(j)) (k-1)} \right)^2,
\end{align*}
where the $\alpha(j) = Y_i(j) - X_i(j)$. Taking expectations on both sides and using the fact that the
$\alpha(j)$ are independent zero mean random variables over a range of
$s_i/(k-1)$, we get
\begin{align}
   \EE\sum_r h_r (a+br)^2 & =  \sum^d_{j=1}  \EE(X_i(j)^2 + \alpha(j)^2)(k-1)^2 \nonumber \\
  & \leq  \norm{X_i}^2_2  \left((k-1)^2+ \frac{d}{2} \right). \label{eq:z1}
\end{align}
We now bound $\sum_r \frac{h_r}{d} \log_2 \frac{d}{h_r}$ in terms of $\sum_r h_r (a+br)^2 $. Let $p_r = h_r/d$ and $\beta = \sum^{k-1}_{r=0} 1/((a+br)^2+\delta)$.
Note that
\begin{align*}
\sum_r p_r \log_2 \frac{1}{p_r}
&= \sum_r p_r \log_2 \frac{1/(((a+br)^2 +\delta)\beta)}{p_r} \\
& + \sum_r p_r \log_2 (((a+br)^2+\delta)\beta) \\
& \leq \sum_r p_r \log_2(((a+br)^2 +\delta)\beta) \\
& \leq \log_2 (\sum_r p_r (a+br)^2 + \delta) + \log_2 \beta,
\end{align*}
where the first inequality follows from the positivity of
KL-divergence.  Choosing $\delta = s^2_i$, yields $\beta \leq 4/s^2_i$ and
hence,
\begin{align}
\sum_r p_r \log_2 \frac{1}{p_r} 
& \leq \log_2 (\sum_r p_r (a+br)^2 + s^2_i) + \log_2(4/s^2_i) \nonumber \\
& = 2 + \log_2 (\sum_r p_r (a+br)^2/ s^2_i + 1). \label{eq:z2}
\end{align}
Combining the above set of equations, we get that the expected communication is bounded by
\begin{align*}
 \EE \left[d \sum^{k-1}_{r=0} \frac{h_r}{d} \log_2 \frac{d}{h_r} + 2 \right] 
& = d\cdot \EE \left[ \sum^{k-1}_{r=0} p_r \log_2 \frac{1}{p_r}\right] +2 \\
& \stackrel{(a)}{\leq}d\cdot \EE \left[ 2 + \log_2 (\sum_r p_r (a+br)^2/ s^2_i + 1) \right] +2  \\
&  \stackrel{(b)}{\leq} d \left( 2 + \log_2 \left( \EE \left[ \sum_r p_r (a+br)^2/s^2_i + 1 \right] \right) \right) +2   \\
& \leq d \left(2 + \log_2\left(\frac{(k-1)^2}{2d }+ \frac{5}{4} \right) \right) +2 ,
\end{align*}
where $(a)$ follows from Equation~\eqref{eq:z2} and $(b)$ follows from Jensen's inequality. The last inequality follows from Equation~\eqref{eq:z1}.
\end{proof}

Thus if $k = \sqrt{d} + 1$, the communication complexity is $\cO(nd)$
and the MSE is $\cO(1/n)$.

\section{Communication MSE trade-off}
\label{sec:sampling}

In the above protocols, all the clients transmit and hence the
communication cost scales linearly with $n$. Instead, we show that any
of the above protocols can be combined by client
sampling to obtain trade-offs between the MSE and the
communication cost. Note that similar analysis also holds for  sampling the coordinates.

Let $\pi$ be a protocol where the mean estimate is of the form:
\begin{equation}
  \label{eq:average_protocol}
\hat{\bar{X}} = R^{-1} \frac{1}{n} \sum^n_{i=1} Y_i.
\end{equation}
All three protocols we have discussed are of this form. Let
$\pi_p$ be the protocol where each client participates independently
with probability $p$. The server estimates $\bar{X}$ by
\[
\hat{\bar{X}}_{\pi_{p}} = R^{-1} \cdot \frac{1}{np} \sum_{i \in S} Y_i,
\]
where $Y_i$s are defined in the previous section and $S$ is the set of
clients that transmitted.

\begin{Lemma}
  \label{lem:sampling}
  For any set of vectors $X^n$ and protocol $\pi$ of the form  Equation~\eqref{eq:average_protocol}, its sampled version
  $\pi_p$ satisfies
  \[
  \cE(\pi_p, X^n) = \frac{1}{p} \cdot   \cE(\pi, X^n)  + \frac{1-p}{np} \sum^n_{i=1} \norm{X_i}^2_2.
  \]
  and
  \[
\cC(\pi_p, X^n) = p \cdot \cC(\pi, X^n).
  \]
\end{Lemma}
\begin{proof}
  The proof of communication cost follows from Lemma~\ref{lem:comrsb}
  and the fact that in expectation,  $np$ clients transmit.
We now bound the MSE.  Let $S$ be the set of clients
that transmit. The error $\cE(\pi_{p}, X^n) $ is
\begin{align*}
 \EE \left[ \norm{\hat{\bar{X}} -
    \bar{X}}^2_2 \right]  &= \EE \left[ \norm{\frac{1}{np} \sum_{i
      \in S} R^{-1}Y_i - \bar{X}}^2_2 \right] \\ &
{=}\EE \left[
  \norm{\frac{1}{np}\sum_{i \in S} X_i - \bar{X}}^2_2 + \frac{1}{n^2p^2} \norm{\sum_{i \in S}
    (R^{-1}Y_i - X_i)}^2_2\right],
\end{align*}
where the last equality follows by observing that $R^{-1}Y_i - X_i$ are
independent zero mean random variables and hence for any $i$,
$\EE[(R^{-1}Y_i - X_i)^T (\sum_{i \in S} X_i -
  \bar{X}) ] = 0$. The first term can be bounded as
\begin{align*}
   \EE  \norm{\frac{1}{np}\sum_{i \in S} X_i - \bar{X}}^2_2 &
 =\frac{1}{n^2} \sum^n_{i=1} \EE \norm{\frac{1}{p} X_i\indic_{i \in S} - X_i}^2_2  \\
  & =\frac{1}{n^2} \sum^n_{i=1} \left(p \frac{(1-p)^2}{p^2} \norm{X_i}^2_2 +  (1-p) \norm{X_i}^2_2 \right) \\
  & = \frac{1-p}{np} \cdot \frac{1}{n} \sum^n_{i=1} \norm{X_i}^2_2. 
\end{align*}
Furthermore, the second term can be bounded as
\begin{align*}
 \EE \left[\frac{1}{n^2p^2} \norm{\sum_{i \in S}
    (R^{-1}Y_i - X_i)}^2_2 \right] 
  & \stackrel{(a)}{=} \frac{1}{n^2p^2}\sum_{i \in S}\EE \left[ \norm{ 
    (R^{-1}Y_i -X_i)}^2_2 \right] \\
  & = \frac{1}{n^2p^2}\sum^n_{i=1} \EE \left[\norm{ 
    (R^{-1}Y_i -X_i)}^2_2 \indic_{i \in S} \right] \\
& = \frac{1}{n^2 p} \sum^n_{i=1} \EE\left[ \norm{R^{-1}Y_i - X_i}^2_2
    \right] \\
  & = \frac{1}{n^2 p}  \EE\left[ \norm{\sum^n_{i=1}(R^{-1}Y_i - X_i)}^2_2
  \right]   = \frac{1}{p} \cE(\pi, X^n)
\end{align*}
where the last equality follows from the assumption that $\pi$'s mean
estimate is of the form~\eqref{eq:average_protocol}.  $(a)$ follows
from the fact that $R^{-1}Y_i - X_i$ are independent zero mean random
variables.  
\end{proof}
Combining the above lemma with Theorems~\ref{thm:sk} and~\ref{thm:arith}, and choosing $k =\sqrt{d}+1$ results in the following.
\begin{Corollary}
  \label{cor:upper}
  For every $c \leq nd(2 + \log_2(7/4))$, there exists a protocol $\pi$ such that
  $\cC(\pi,S^d) \leq c$ and 
  \[
  \cE(\pi, S^d) = \cO \left( \min \left(1, \frac{d}{c} \right)\right).
  \]
\end{Corollary}
\section{Lower bounds}
\label{sec:lower}
The lower bound relies on the lower bounds on distributed
statistical estimation due to~\cite{ZhangYDJW13}.

\begin{Lemma}[\cite{ZhangYDJW13} Proposition 2]
  \label{lem:zhang}
  There exists a set of distributions $\cP_d$ supported on
  $\left[-\frac{1}{\sqrt{d}}, \frac{1}{\sqrt{d}} \right]^d$ such that
  if any centralized server wishes to estimate the mean of the
  underlying unknown distribution,
  then for any independent protocol $\pi$
    \[
\max_{p_d \in \cP_d} \EE\left[\norm{\theta(p_d) - \hat{\theta}_\pi}^2_2\right] \geq t \min \left(1, \frac{d}{\cC(\pi)} \right),
\]
where $\cC(\pi)$ is the communication cost of the protocol, $\theta(p_d)$ is the mean of $p_d \in \cP_d$,  and
$t$ is a positive constant.
\end{Lemma}
\begin{Theorem}
  \label{thm:lower}
Let $t$ be the constant in Lemma~\ref{lem:zhang}.  For every $c \leq ndt/4$ and $ n \geq 4/t$,
  \[
  \cE(\Pi(c), S^d) \geq   \frac{t}{4} \min \left(1, \frac{d}{c} \right).
\]
\end{Theorem}
\begin{proof}
Given $n$ samples from the underlying distribution where each sample
belongs to $S^d$, it is easy to see that 
\[
 \EE\norm{\theta(p_d)- \hat{\theta}(p_d)}^2_2 \leq \frac{1}{n},
 \]
 where $\hat{\theta}(p_d)$ is the empirical mean of the observed samples.
 Let $\cP_d$ be the set of distributions in Lemma~\ref{lem:zhang}.
 Hence for any protocol $\pi$ there exists a distribution $p_d$ such that
   \begin{align*}
       \EE\norm{\hat{\theta}(p_d) - \hat{\theta}_\pi}^2_2 
     & \stackrel{(a)}{\geq}
    \frac{1}{2} \EE\norm{\theta (p_d) - \hat{\theta}_\pi}^2_2-
    \EE\norm{\theta(p_d)- \hat{\theta}(p_d)}^2_2\\
    & \stackrel{(b)}{\geq} \frac{t}{2}
    \min \left(1, \frac{d}{\cC(\pi)} \right) - \frac{1}{n} 
     \stackrel{(c)}{\geq}
    \frac{t}{4} \min \left(1, \frac{d}{\cC(\pi)} \right),
   \end{align*}
   $(a)$ follows from the fact that $2(a-b)^2 + 2(b-c)^2 \geq
   (a-c)^2$. $(b)$ follows from Lemma~\ref{lem:zhang} and $(c)$
   follows from the fact that $\cC(\pi, S^d) \leq ndt/4$ and $n \geq
   4/t$.
  \end{proof}
\vspace{-0.2cm}
Corollary~\ref{cor:upper} and Theorem~\ref{thm:lower} yield Theorem~\ref{thm:main}.
We note that the above lower bound holds only for communication cost $c < \cO(nd)$. Extending the results for larger values of $c$ remains an open problem.

At a first glance it may appear that combining structured random matrix and variable length encoding may improve the result asymptotically, and therefore
violates the lower bound. However, this is not true. 
Observe that variable length coding $\pi_{svk}$ and stochastic rotated quantization $\pi_{srk}$ use different aspects of the data: the variable length coding  uses the fact that bins with large values of index $r$ are less frequent. Hence, we can use fewer bits to encode frequent bins and thus improve communication. In this scheme bin-width ($s_i/(k-1)$) is $\sqrt{2} ||X_i||_2/(k-1)$.  
Rotated quantization uses the fact that rotation makes the min and max closer to each other and hence we can make bins with smaller width. In such a case, all the bins become “more or less equally likely” and hence variable length coding does not help. In this scheme bin-width ($s_i/(k-1)$) is $ (Z^{\max}_i - Z^{\min}_i)/(k-1) \approx ||X_i||_2 (\log d) /(k d)$, which is much smaller than bin-width for variable length coding. 
Hence variable length coding and random rotation cannot be used simultaneously.  

\begin{figure}[t]
  \centering
    \includegraphics[width=0.5\textwidth]{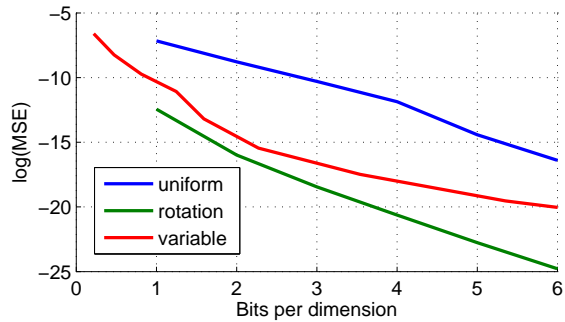}
    \caption{
       Distributed mean estimation on data generated from a Gaussian distribution. } \label{fig:bits}
  \end{figure}
\input{experiments}

\section*{Acknowledgments}
We thank Jayadev Acharya, Keith Bonawitz, Dan
Holtmann-Rice, Jakub Konecny, Tengyu Ma, and Xiang Wu for helpful comments and discussions.

\bibliographystyle{plain}
\bibliography{ref}
\appendix
\section{Proof of Lemma~\ref{lem:maxz}}
\label{app:maxz}
  The equality follows from the symmetry in $HD$. To prove the upper bound, observe that
  \[
  \EE\left[(\maxZi)^2\right] = \Var
  \left(\maxZi \right) +
  \left(\EE\left[\maxZi\right]\right)^2.
  \]
  Let $D(j)$ be the $j^{\text{th}}$ diagonal entry of $D$.
  To bound the first term observe that $\maxZi$
  is a function of $d$ independent random variables $D(1),D(2),\ldots
  D(d)$. Changing $D(j)$ changes the $\maxZi$ by
  at most $\frac{2X_i(j)}{\sqrt{d}}$. Hence, applying Efron-Stein
  variance bound~\cite{EfronS81} yields 
  \[
  \vspace{-0.1cm}
\Var \left(\maxZi \right)\leq \sum^d_{j=1}
\frac{4X^2_i(j)}{2d} = \frac{2\norm{X_i}^2_2}{d}.
  \]
To bound the second term, observe that for every $\beta > 0$,
  \begin{align*}
  \vspace{-0.1cm}
    \beta \maxZi & = \log \exp \left( \beta
    \maxZi \right) \leq \log
    \left(\sum^d_{j=1} e^{\beta Z_i(j)}\right).
  \end{align*}
  Note that $Z_i(k) =\frac{1}{\sqrt{d}} \sum^d_{j=1} D(j) H(k,j)
  X_i(j)$. Since the $D(j)$'s are Radamacher random variables and $|H(k,j)| =1$ for all $k,j$, the
  distributions of $Z_i(k)$ is same for all $k$. Hence by Jensen's
  inequality,
\begin{align*}
  & \EE\left[ \maxZi\right]
   \leq \frac{1}{\beta}
  \EE\left[ \log \left(\sum^d_{j=1} e^{\beta Z_i(j)}\right) \right] \\
  & \leq \frac{1}{\beta} \log \left(\sum^d_{j=1} \EE[ e^{\beta Z_i(j)}]\right) 
  = \frac{1}{\beta} \log \left(d \EE[ e^{\beta Z_i(1)}]\right).
\end{align*}
Since $Z_i(1) =\frac{1}{\sqrt{d}} \sum^d_{j=1} D(j) X_i(j)$,
\begin{align*}
  \EE[ e^{\beta Z_i(1)}] & = \EE \left[ e^{\frac{\beta\sum_j D(j)
        X_i(j)}{\sqrt{d}} } \right]  \stackrel{(a)}{=}
  \prod^d_{j=1} \EE \left[ e^{\frac{\beta D(j) X_i(j)}{\sqrt{d}} }
    \right]\\ & = \prod^d_{j=1} \frac{e^{-\beta X_i(j) /\sqrt{d}}
    +e^{\beta X_i(j) /\sqrt{d}} }{2} \\ & \stackrel{(b)}{\leq}
  \prod^d_{j=1} e^{\beta^2 X^2(j) /2d}  = e^{\beta^2
    \norm{X_i}^2_2/2d},
\end{align*}
where $(a)$ follows from the fact that the $D(i)$'s are independent and
$(b)$ follows from the fact that $e^a + e^{-a} \leq 2 e^{a^2/2}$ for
any $a$.  Hence,
\[
\EE[ \maxZi] \leq \min_{\beta \geq 0} \frac{\log d}{\beta}
+ \frac{\beta \norm{X_i}^2_2}{2d} \leq \frac{2\norm{X_i}_2 \sqrt{\log
    d}}{\sqrt{2d}}. \qedhere
\]
\end{document}

%% file: def_t.tex
\definecolor{darkgreen}{rgb}{0,0.4,0.0}

\newcommand{\maxZi}{Z^{\max}_i}
\newcommand{\minZi}{Z^{\min}_i}

\newcommand{\indic}{\mathbbm{1}}
\newcommand{\Var}{\mathrm{Var}}

\newcommand{\ed}{\stackrel{\mathrm{def}}{=}}

\newcommand{\EE}{\mathbb{E}}

\newcommand{\RR}{\mathbb{R}}
\newcommand{\norm}[1]{\left\lvert\left\lvert#1\right\rvert\right\rvert}

\newcommand{\cO}{\mathcal{O}}
\newcommand{\cP}{\mathcal{P}}

\newcommand{\cE}{\mathcal{E}}
\newcommand{\cC}{\mathcal{C}}

\newcommand{\tcO}{\tilde{\mathcal{O}}}

\newcommand{\iid}{\emph{i.i.d.\@\xspace}}

\newtheorem{Theorem}{Theorem}
\newtheorem{Corollary}{Corollary}
\newtheorem{Lemma}{Lemma}

\newcommand{\ignore}[1]{}%

\usepackage{prettyref}
\newrefformat{eq}{(\ref{#1})}
\newrefformat{thm}{Theorem~\ref{#1}}
\newrefformat{th}{Theorem~\ref{#1}}
\newrefformat{chap}{Chapter~\ref{#1}}
\newrefformat{sec}{Section~\ref{#1}}
\newrefformat{algo}{Algorithm~\ref{#1}}
\newrefformat{fig}{Fig.~\ref{#1}}
\newrefformat{tab}{Table~\ref{#1}}
\newrefformat{rmk}{Remark~\ref{#1}}
\newrefformat{clm}{Claim~\ref{#1}}
\newrefformat{def}{Definition~\ref{#1}}
\newrefformat{cor}{Corollary~\ref{#1}}
\newrefformat{lmm}{Lemma~\ref{#1}}
\newrefformat{prop}{Proposition~\ref{#1}}
\newrefformat{pr}{Proposition~\ref{#1}}
\newrefformat{app}{Appendix~\ref{#1}}
\newrefformat{apx}{Appendix~\ref{#1}}
\newrefformat{ex}{Example~\ref{#1}}
\newrefformat{exer}{Exercise~\ref{#1}}
\newrefformat{soln}{Solution~\ref{#1}}

\ignore{

}

%% file: experiments.tex
\begin{figure}[t]
\vspace{-0.2cm}
\centering
\subfigure[\texttt{MNIST} $k=16$]{\includegraphics[width=0.23\textwidth]{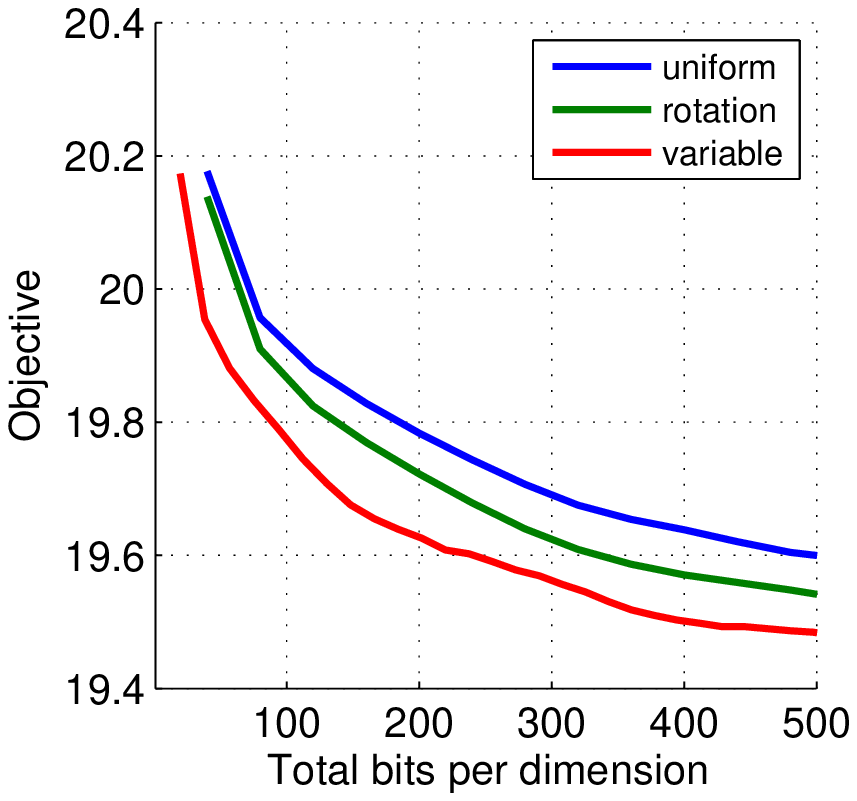}}
\hspace{-0.1cm}
\subfigure[\texttt{MNIST} $k=32$]{\includegraphics[width=0.23\textwidth]{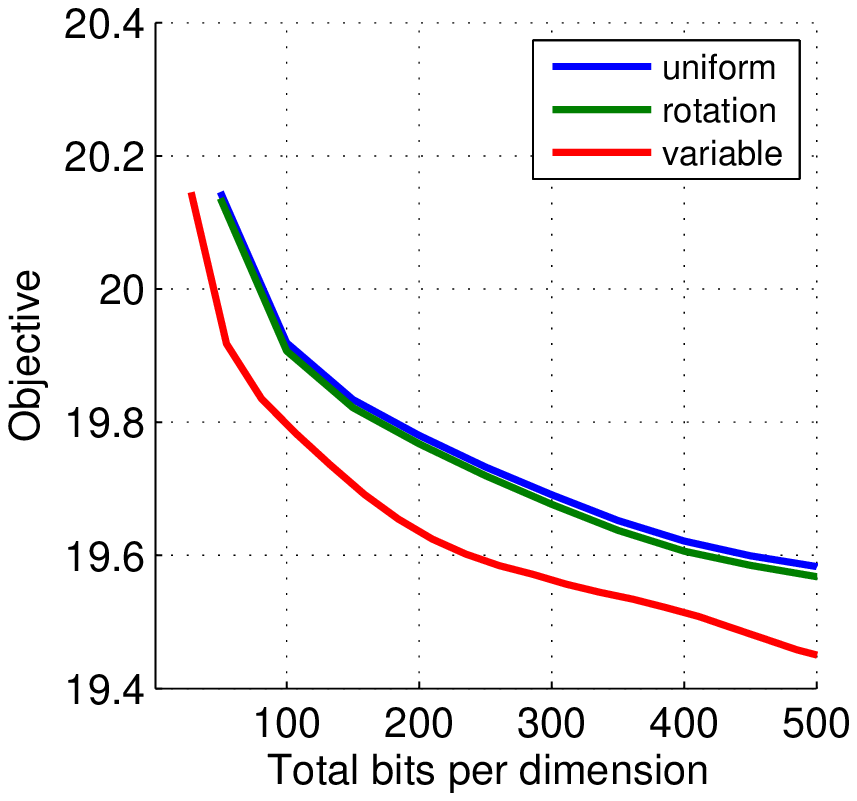}}
\hspace{-0.1cm}
\vspace{-0.2cm}
\subfigure[\texttt{CIFAR} $k=16$]{\includegraphics[width=0.23\textwidth]{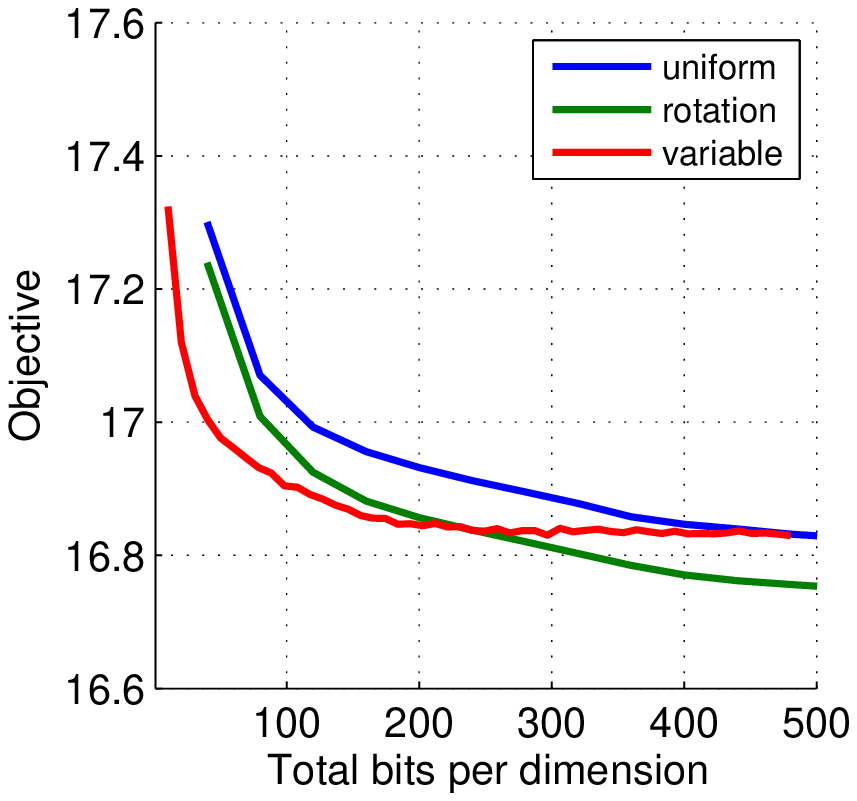}}
\hspace{-0.1cm}
\subfigure[\texttt{CIFAR} $k=32$]{\includegraphics[width=0.23\textwidth]{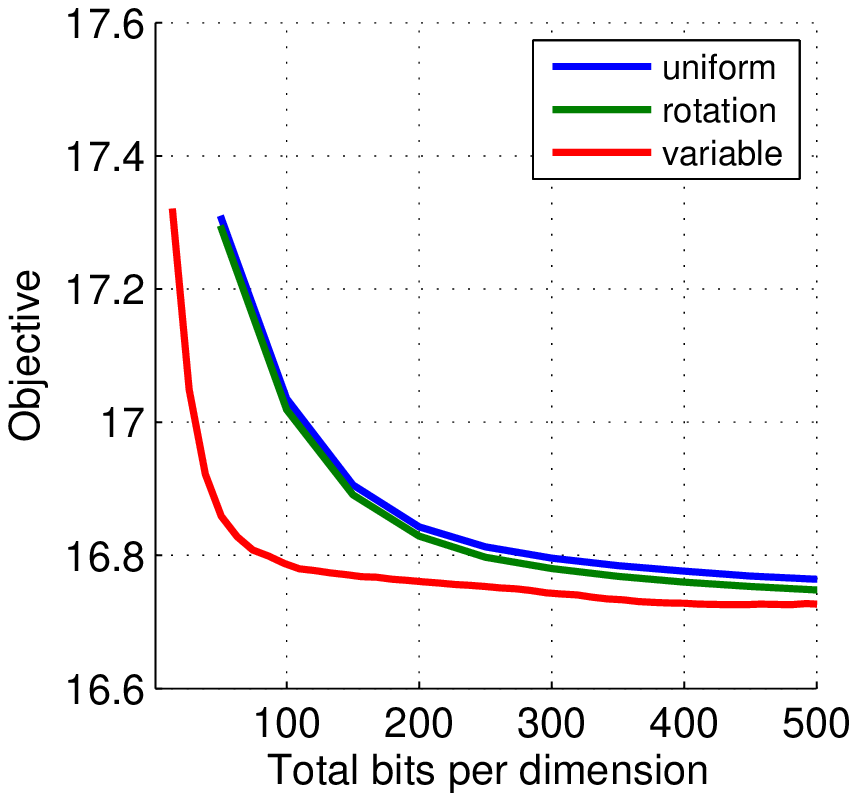}}
\hspace{-0.1cm}
\caption{Lloyd's algorithm with different types of quantizations. uniform: stochastic $k$-level, rotation: stochastic rotated, variable: variable-length coding. Here we test two settings: 16 quantization levels and 32 quantization levels. The x-axis is the averaged number of bits sent for each data dimension, i.e., communication cost (this scales linearly to the number of iterations), and the y-axis is the global objective of Lloyd's algorithm. }
\label{fig:kmeans}
\end{figure}
\begin{figure}[t]
\vspace{-0.2cm}
\centering
\subfigure[\texttt{MNIST} $k=16$]{\includegraphics[width=0.23\textwidth]{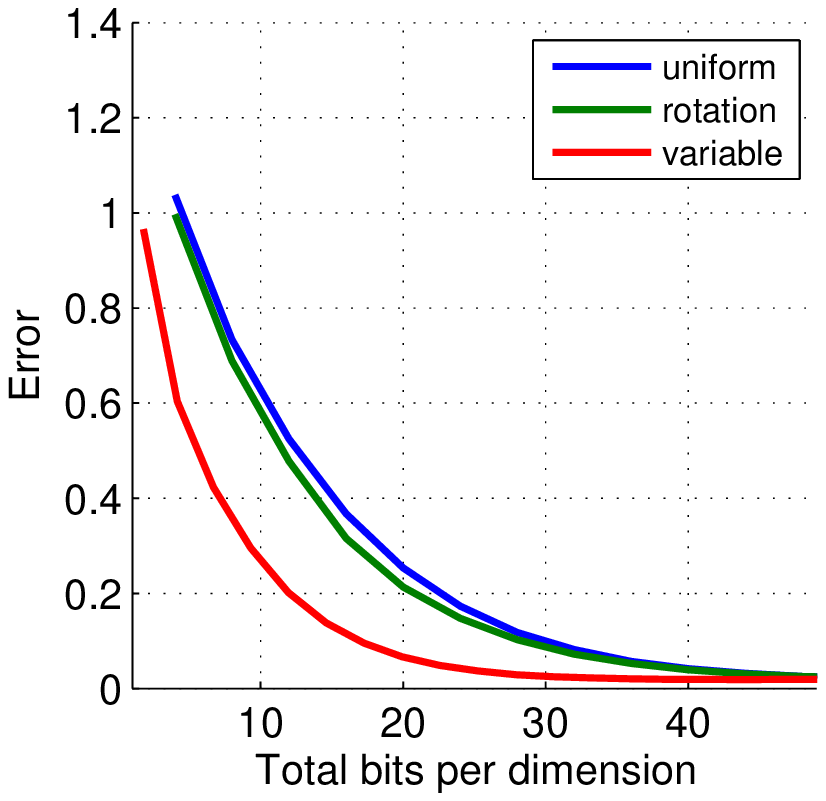}}
\hspace{-0.1cm}
\subfigure[\texttt{MNIST} $k=32$]{\includegraphics[width=0.23\textwidth]{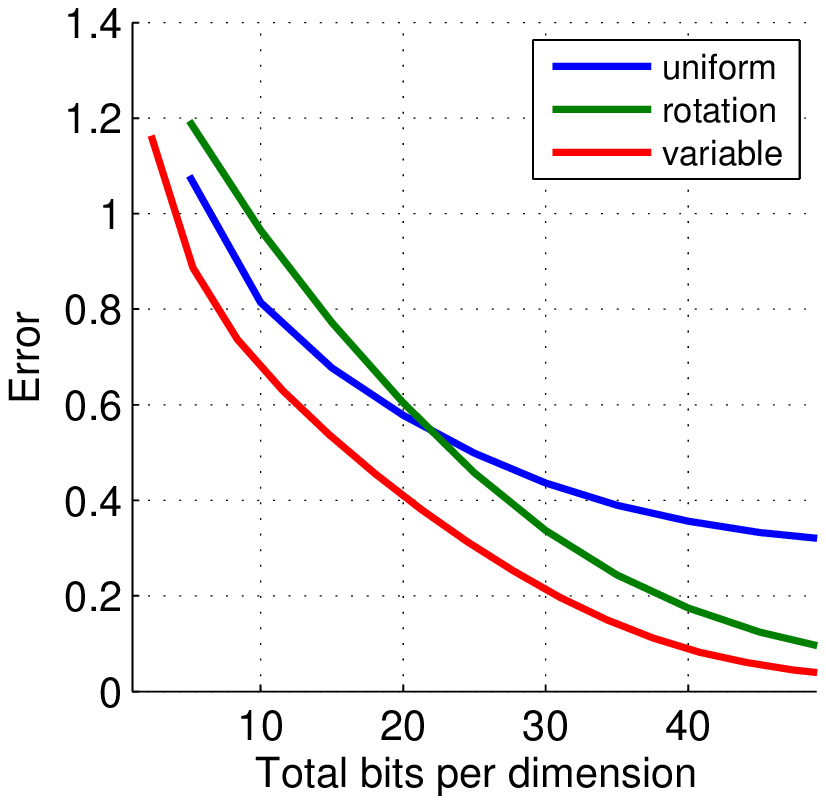}}
\hspace{-0.1cm}
\vspace{-0.2cm}
\subfigure[\texttt{CIFAR} $k=16$]{\includegraphics[width=0.23\textwidth]{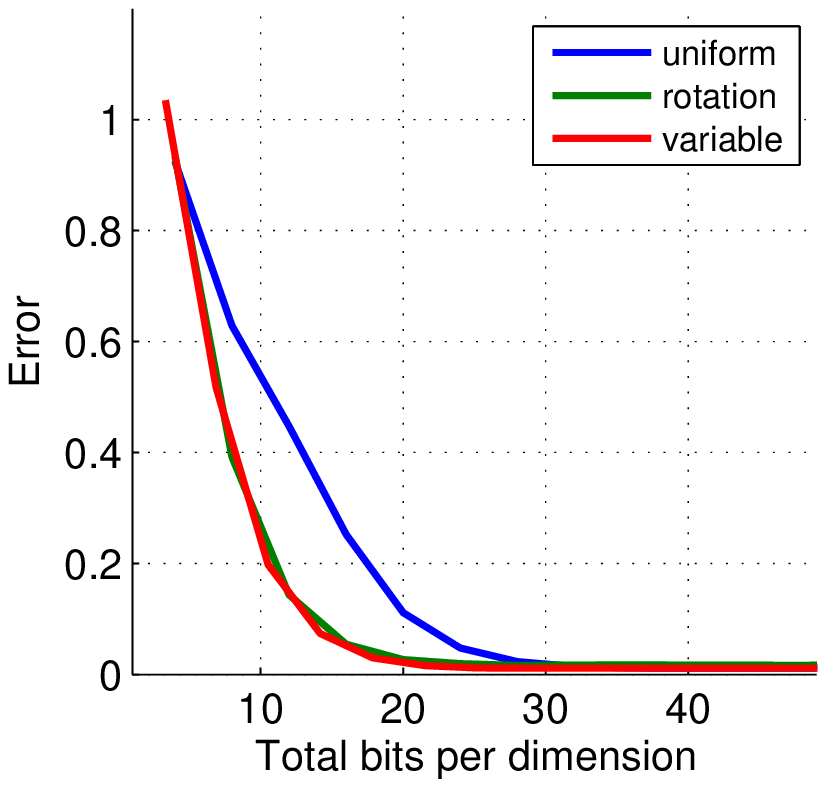}}
\hspace{-0.1cm}
\subfigure[\texttt{CIFAR} $k=32$]{\includegraphics[width=0.23\textwidth]{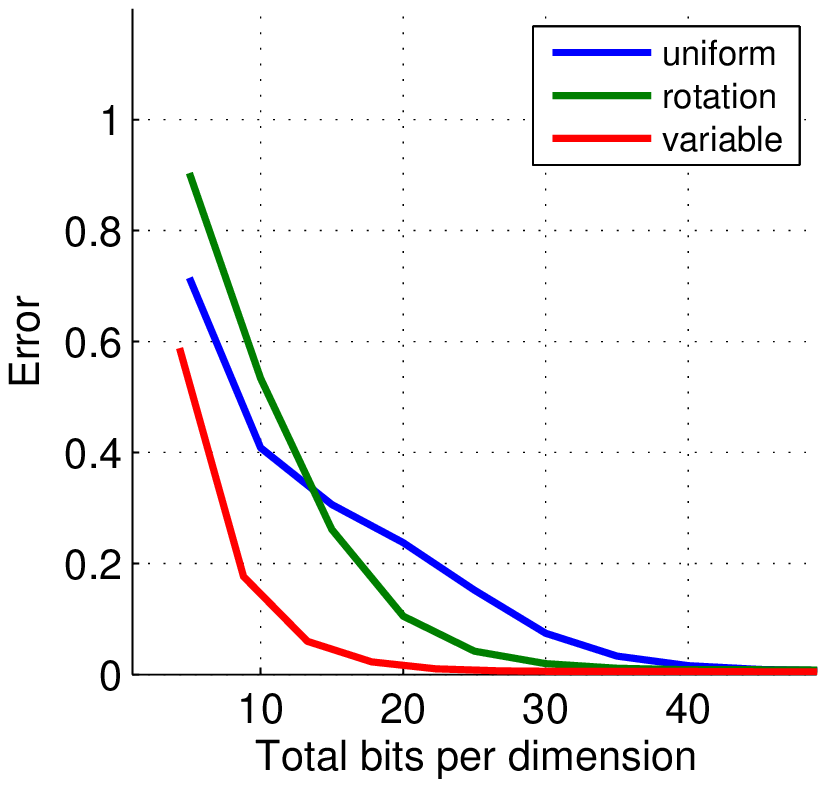}}
\hspace{-0.1cm}
\caption{Power iteration with different types of quantizations. uniform: stochastic $k$-level, rotation: stochastic rotated, variable: variable-length coding. Here we test two settings: 16 quantization levels and 32 quantization levels. The x-axis is the averaged number of bits sent for each data dimension, i.e., communication cost (this scales linearly to the number of iterations), and the y-axis is the $\ell_2$ distance between the computed eigenvector and the ground-truth  eigenvector. }
\label{fig:pca}
\end{figure}
\section{Practical considerations and applications}
\label{sec:experiments}

Based on the theoretical analysis, the variable-length coding method provides the lowest quantization error asymptotically when using a constant number of bits. Stochastic rotated quantization in practice may be preferred due to the (hidden) constant factors of the variable length code methods.
For example, considering quantizing a single vector $[-1, 1, 0, 0]$, stochastic rotated quantization can use 1 bit per dimension and achieves zero error: the rotated vector has only two values $0$, $2$ or $0$, $-2$. 
We further note that the rotated quantization is preferred when applied on ``unbalanced'' data, due to the fact that the rotation can correct the unbalancedness. We demonstrate this by generating a dataset where the value of the last feature dimension entry is much larger than others.
We generate 1000 datapoints each with 256 dimensions. The first 255 dimensions are generated i.i.d. from $N(0, 1)$, and the last dimension is generated from $N(100, 1)$. As shown in Figure \ref{fig:bits}, the rotated stochastic quantization has the best performance. The improvement is especially significant for low bit rate cases. 

We demonstrate two applications in the rest of this section. The experiments are performed on the MNIST ($d = 1024$) and CIFAR ($d = 512$) datasets.

\textbf{Distributed Lloyd's algorithm.} In the distributed Lloyd's (k-means) algorithm, each client has access to a subset of data points. In each iteration, the server broadcasts the cluster centers to all the clients. Each client updates the centers based on its local data, and sends the centers back to the server. The server then updates the centers by computing the weighted average of the centers sent from all clients. In the quantized setting, the client compresses the new centers before sending to the server. This saves the uplink communication cost, which is often the bottleneck of distributed learning\footnote{In this setting, the downlink is a broadcast, and therefore its cost can be reduced by a factor of $O(n / \log n )$ without quantization, where $n$ is the number of clients.}. 
We set both the number of centers and number of clients to 10. Figure \ref{fig:kmeans} shows the result. 

\textbf{Distributed power iteration.} Power iteration is a widely used method to compute the top eigenvector of a matrix. In the distributed setting, each client has access to a subset of data. In each iteration, the server broadcasts the current estimate of the eigenvector to all clients. Each client then updates the eigenvector based on one power iteration on its local data, and sends the updated eigenvector back to the server. The server updates the eigenvector by computing the average of the eigenvectors sent by all clients. 
Similar to the above distributed Lloyd's algorithm, in the quantized setting, the client compresses the estimated eigenvector before sending to the server. Figure \ref{fig:pca} shows the result.
The dataset is distributed over $100$ clients.

For both of these applications, variable-length coding achieves the lowest quantization error in most of the settings. Furthermore, for low-bit rate, stochastic rotated quantization is competitive with variable-length coding.